\documentclass{article}

\usepackage{microtype}
\usepackage{graphicx}
\usepackage{subfigure}
\usepackage[title]{appendix}
\usepackage{booktabs} 
\graphicspath{{./figs/}}
\usepackage{hyperref}

\usepackage[margin=1in]{geometry}

\usepackage{amsmath}
\usepackage{amssymb}
\usepackage{mathtools}
\usepackage{amsthm}
\usepackage[capitalize,noabbrev]{cleveref}

\usepackage[utf8]{inputenc} 
\usepackage[T1]{fontenc}    
\usepackage{hyperref}       
\usepackage{url}            
\usepackage{booktabs}       
\usepackage{amsfonts}       
\usepackage{nicefrac}       
\usepackage{microtype}      
\usepackage{xcolor, graphicx}         
\usepackage{hyperref}       
\hypersetup{
    colorlinks=true,
    linkcolor=blue,
    filecolor=magenta,      
    urlcolor=cyan,
    pdftitle={Overleaf Example},
    pdfpagemode=FullScreen,
    }
\usepackage{amsmath}
\usepackage{amsthm}
\usepackage{subfigure}
\newtheoremstyle{break}
  {\topsep}{\topsep}%
  {\itshape}{}%
  {\bfseries}{}%
  {\newline}{}%
\newtheorem{theorem}{Theorem}[section]
\newtheorem*{theorem*}{Theorem}
\newtheorem{definition}{Definition}[section]
\newtheorem{proposition}{Proposition}[section]
\newtheorem{lemma}{Lemma}[section]
\newtheorem*{lemma*}{Lemma}

\theoremstyle{remark}

\usepackage{mathtools}
\usepackage{enumitem}
\usepackage{amsmath,centernot}
\usepackage{xspace}
\usepackage{booktabs}
\usepackage{cleveref}
\usepackage{wrapfig}
\usepackage{algorithm}
\usepackage{algpseudocode}
\usepackage{bbm}

\DeclarePairedDelimiterX{\infdivx}[2]{(}{)}{%
  #1\;\delimsize\|\;#2%
}

\newlength{\inputlength}
\def\thickhline{\noalign{\hrule height.8pt}}

\title{Reframing Data Value for Large Language Models Through the Lens of Plausibility}

\author{%
	Mohamad Rida Rammal \\
University of California, Los Angeles\\
\texttt{ridarammal@g.ucla.edu} \\
	\and
		Ruida Zhou \\
	University of California, Los Angeles\\
	\texttt{ruida@g.ucla.edu} \\
	\and
		Suhas Diggavi \\
	University of California, Los Angeles\\
	\texttt{suhas@ee.ucla.edu} \\
}

\begin{document}

\date{}
\maketitle

\begin{abstract}
  \noindent Data valuation seeks to answer the important question, "How much is this data worth?" Existing data valuation methods have largely focused on discriminative models, primarily examining data value through the lens of its utility in training. However, with the push for ever-larger language models, relying on valuation methods that require training becomes increasingly expensive and dependent on specific techniques. We propose an alternative perspective on the data value problem for language models, centering around the plausibility of the data. We posit that data holds lesser value if it can be plausibly generated by the model itself. Starting from some intuitive criteria that align with our notions of valuable data, we develop a novel value function that is computationally tractable and derived from first principles with provable properties. We conduct a theoretical analysis of our value function and evaluate it across multiple scenarios and datasets.
\end{abstract}

\section{Introduction}
An ongoing challenge in language modeling is obtaining high-quality data to evaluate and improve model performance. However, the proliferation of large language models (LLMs) has raised concerns about the use of copyrighted material without proper authorization from the original creators. As data owners restrict access to data that was once publicly available, a natural question arises: How can we determine the inherent value of a dataset to a learning model? To explore this, we consider a hypothetical scenario involving two individuals: Alice and Bob. Alice owns a language model, while Bob possesses a dataset that Alice might want to acquire (or purchase) from him. We aim to assess the value of this dataset to Alice. Note that this differs from the standard way the data valuation problem is posed, where the goal is to fairly allocate credit for the validation performance of a learning algorithm among the training data.\\ 

One approach to determining acquisition value involves training the model. This approach typically requires training the model (often multiple times) with the dataset and assessing its impact on performance \cite{pmlr-v97-ghorbani19c,jiang2021characterizing,pmlr-v89-jia19a,just2023lava,pmlr-v162-wu22j,NEURIPS2021_59a3adea}. However, training large language models can be prohibitively expensive and time-consuming, and the results may vary significantly depending on the training method. A dataset might be considered valuable with one algorithm but less so with another due to variability in outcomes and the specific task on which the model is evaluated. In this work, we propose an intuitive methodology that quantifies the value of data for language models through statistical means rather than through optimization or training. \\ 

The main intuition behind our approach arises from replacing the challenging question, "what makes data valuable?" with the closely related but more manageable "What data is not worth acquiring?" For a generative model, such as a large language model, there is a natural answer to the latter question: data that can be generated by the model itself. Referring back to the earlier scenario, if Bob’s dataset was generated by Alice’s model, we would assign a low “acquisition” value to the data, as Alice can generate it herself without needing to acquire it. We posit that data holds lesser value if it can be plausibly generated by the model itself. Our focus is on determining how difficult it would be for Alice to generate this dataset on her own. The remaining question is how to quantitatively assess this plausibility. \\


In this framework, we view the language model as a token predictor denoted by $p(\cdot|x_{1:i})$, where $x_{1:i}=(x_{1},\ldots,x_{i})$ represents the history or context of tokens used to predict the next token $x_{i+1}$. Here, each token is drawn from a vocabulary $\mathcal{V}$. As the owners of the language model being used for data valuation, we have access to this probability distribution $p(\cdot|x_{1:i})$. Consequently, the plausibility question revolves around assessing the statistical disparity between the model and the data. The greater this disparity, the more valuable the data. This perspective aligns with the classical framework of distribution testing, where the goal is to determine whether a given sequence of data is generated from a specified model (see \cite{WOLFER202385} and references therein). \\

However, language models present two significant challenges: (i) the state space is massive; the alphabet size of tokens (the number of possible values tokens can take), $\left\lvert \mathcal{V}\right\rvert$, is typically in the tens of thousands. A typical context length $L$, the maximum number of tokens the model can remember at any given time, can also be somewhere in the several thousands, necessitating methods which are computationally efficient. Moreover, (ii) to provide any performance guarantees with such a large state space, the length of data sequences needs to be enormous, often much larger than $\left\lvert \mathcal{V}\right\rvert^L$. This necessitates methods capable of operating effectively with much smaller data sequences, as obtaining single datapoint sequences with these lengths is unrealistic. \\

To address this, we develop a measure of value using Rosenblatt's transformation, which is a mapping that converts \textit{continuous} random vectors from any distribution into uniform random variables \cite{rosenblatt1952remarks}. Specifically, suppose $Y=(Y_1,\ldots,Y_d)\sim F$ is a continuous random vector distributed according to the cumulative distribution function (CDF) $F$, then $Z_i=F(Y_i\mid Y_1,\ldots,Y_{i-1})\sim \mathcal{U}$ is a sequence of independent and independent and identically distributed random variables with the standard uniform distribution over $(0,1)$\cite{rosenblatt1952remarks}. We first convert our discrete input tokens $x_1,\ldots, x_n$ into continuous variables $y_1,\ldots,y_n$ since Rosenblatt's transformation is only valid for continuous random variables. We then apply this transform using a continuous representation of probability distribution $p(\cdot| x_{1:i})$ output by the model. We prove that if the sequence of tokens $x_i$ were indeed generated from the model, then the obtained sequence of variables after transformation $Z_i$ are independent and uniform. Using this we develop an efficient method for data value for language models by using a distribution distance function between the empirical distribution of $Z_1,\ldots,Z_n$ with the i.i.d., $\mathcal{U}$ distribution.\footnote{We need to also add an independence test for this; see details in Section \ref{sec:main}} \\

\textbf{Contributions.}
Our main contributions are summarized as follows:
\vspace{0.2cm}
\begin{itemize}[topsep=0pt, left=0pt] 
    \item We introduce a novel value function: the Uniform-Marginal and Independence (UMI) value function, constructed from the first principles of a desirable value function. The UMI value function is derived from Rosenblatt’s transformation, which disentangles the challenging task of statistical testing for complex data and language models into two simpler problems: testing for uniform distribution and marginal distribution, and independence testing. 
    \item The proposed UMI value function is not only computationally and statistically efficient but is also firmly grounded in theory. We demonstrate that the $f$-divergence between the marginal and uniform distributions establishes a lower bound for the $f$-divergence between the data model and language model (see Theorem \ref{thm: simple_case} and \ref{thm: markov_case}).
    \item The marginal distribution derived from the UMI calculation serves as an effective visualization tool for datasets. Essentially, the dataset is mapped to an empirical cumulative distribution function, allowing for a visual assessment of its proximity to the $y=x$ baseline (see Figure \ref{fig: cdf_visualize}).
    \item Experiments are conducted to illustrate the effectiveness of the proposed UMI function (see Section \ref{sec: experiments}).
\end{itemize}



\section{Preliminaries} 

\subsection{Setup}
We consider language models that take as input a sequence of tokens $x_{1:n}=(x_1,\ldots,x_n)$ drawn from some vocabulary $\mathcal{V}$. The models then output a probability distribution over the next token in the sequence $p(\cdot \mid x_{1:n})$. State-of-the-art language models typically have a maximum context length $L$, in which case the probability of the current tokens depends only on the previous $L$ tokens. A standard value for $L$ is somewhere between $512$ and $4096$ tokens. 

We consider a dataset $\mathcal{D} = \{d_1,\ldots,d_N\}$ composed of multiple datapoints. Each datapoint consists of a sequence of tokens of possibly variable length. Our goal is to find a measure of value for the dataset $\mathcal{D}$ assuming that our knowledge is represented through a language model, i.e., the transition distribution $p$ it determines. In other words, we would like to find a value function $V:\mathcal{D}\mapsto V(\mathcal{D})$ which maps the dataset to a non-negative number and satisfies a number of desirable properties.

\subsection{Properties of a Desirable Value Function}




We argue that a good value function $V$ should possess several key properties and that it should align with some of our intuitive notions about valuable data. Specifically, we are looking for a function which satisfies the following criteria: \\

\textbf{Additivity:} Given a dataset $\mathcal{D}=\{d_1,\ldots,d_N\}$, the value function $V$ of the dataset should be additive in terms of the data points, \emph{i.e.,} $V(\mathcal{D})=\sum_{i=1}^{n}V(\{d_i\})$. Assuming $V(\{d\})\geq 0$ for all $d$, this would immediately imply that if $\mathcal{D}_1$ and $\mathcal{D}_2$ are two datasets such that $\mathcal{D}_1\subseteq \mathcal{D}_2$, then $V(\mathcal{D}_1) \leq V(\mathcal{D}_2)$. \\

\textbf{Baseline:} We require a quantifiable criterion for determining when a data value reaches zero, establishing a baseline against which we can make comparisons. Additionally, as previously outlined, we would like this to be a statistical property rather than one based on optimization or training. \\

\textbf{Efficient:} $V$ should be computationally feasible even for large vocabularies and context sizes. Moreover, $V$ should prioritize sample efficiency, enabling evaluation for sequences of tokens regardless of length without the necessity for excessively lengthy datapoints.



\subsection{Overview of the Solution}
As previously mentioned, our approach stems from a shift in perspective: instead of focusing on what makes data valuable, we pivot towards identifying the data that does not carry value. In the context of a generative model, data generated internally holds no value or relevance to the model's owner. This then serves as the foundation of our approach. By initially focusing on identifying data generated by the model, we establish a starting point for developing a value function. To identify data generated by the model, we make use Rosenblatt’s probability integral transform, see \eqref{eqn: Rosenblatt-Transform}. \\

In particular, we are given a sequence of tokens $x_1,\ldots,x_n$, where each token maps to a unique integer in the set $\{1,\ldots, \lvert \mathcal{V}\rvert\}$. To prepare for Rosenblatt's transformation, which requires continuous variables, we construct a sequence of continuous variables $\tilde{x}_1,\ldots,\tilde{x}_n$ using the transformation $\tilde{x}_i = x_i - u_i$, where $u_i$ is a standard uniform random variable independent of $x_i$. Our language model provides a sequence of probability mass functions over the next token $p(\cdot |x_{i-L}^{i-1})$, one for each token in the input. We can transform these discrete probability functions provided by the model into continuous representations given by
\begin{equation}
\label{eq:z_eq}
\tilde{F}\left(\tilde{x} \mid x_{i-L}^{i-1}\right) = \sum_{j=1}^{\lfloor \tilde{x} \rfloor} p\left(j\mid x_{i-L}^{i-1}\right) + \left(\tilde{x} -\lfloor \tilde{x}  \rfloor\right)p\left(\lfloor \tilde{x} \rfloor \mid x_{i-L}^{i-1}\right)
\end{equation}
Finally, we evaluate the continuous $y_i$'s at their corresponding continuous distribution functions to obtain the new sequence $z_i = \tilde{F}\left(y_i\mid x_{i-L}^{i-1}\right)$. We prove that if the sequence of tokens were generated by the language model, then the $z_i$'s must necessarily be independent and identically distributed $\mathcal{U}$ random variables (see Theorem \ref{thm: multivariate_pit}). \\

To obtain our value measure, we compare the actually obtained distribution of the $z_i$'s to what their distribution would have been generated by the model i.e., the uniform distribution. Specifically, let $G$ be the distribution of the $z_i$'s, then we compute $D_f(G\|\mathcal{U})$, where $D_f(G\|\mathcal{U})$ is the $f$-divergence between the two distributions  $G$ and $\mathcal{U}$.\footnote{$f$-divergence is a generalization of Kullback-Leibler divergence \cite{Cover2006}; see Definition \ref{def:fDiv}. For numerics, we use the more familiar Kullback-Leibler divergence.} Moreover, since Theorem \ref{thm: multivariate_pit} guarantees both uniformity and independence of the $z_i$'s, we employ independence testing to address edge cases where the $z_i$'s are uniform but lack independence.


\subsection{Alternatives}
In this section, we explore two other potential candidates we have considered for our value function, namely one centered on compression and the other on identity testing. We explain why we chose our method over these alternatives and discuss the insights gleaned from them that steered us towards our proposed approach. \\

\textbf{Compression:} In information theory, prediction is often viewed from the lens of compression. If you can compress some data down to a small size, then it likely does not contain much new information beyond what you already knew (and vice versa). In particular, arithmetic coding is a technique which attains near-optimal compression rates when coupled with an accurate model
(distribution) of the data. Large language models, coupled with arithmetic coding, have been shown to excel at compressing text data\cite{delétang2024language, valmeekam2023llmzip}. The compression approach provides a concrete measure of data value: compress the dataset with arithmetic coding and use the compressed size (in bits or nats) as the measure. \\

Let $x_{1:n}=(x_1,x_2\ldots,x_n)$ be a token sequence generated by a distribution $q$ over $\mathcal{V}^n$. It can be shown that any compression scheme corresponds to a distribution $r$ over $\mathcal{V}^n$, where the number of nats used to compress $x_{1:n}$ is $1/r(x_{1:n})$. A standard result asserts that the expected compressed data size $L$ per input token is bounded below by $\mathcal{H}(q) + \mathrm{KL}(q\| r)$, where $\mathcal{H}$ is the entropy rate and $\mathrm{KL}(q\| r)$ is the Kullback-Leibler divergence between $q$ and $r$ \cite{Cover2006}. \\

Therefore, defining the value function as $V(\mathcal{D}) = L - \mathcal{H}(q)$ ensures non-negativity, additivity, and a well-defined baseline since $V(\cdot) = 0$ when $r=q$. However, computing $\mathcal{H}(q)$, the true entropy rate of the token sequence $X^n$, is non-trivial for large alphabet sizes.\footnote{Note that compressing using a language model is like compressing with some $r$, gives an upper bound to the unknown entropy rate.} In particular, the min-max regret analysis of universal data compression shows that the estimation gap goes down as $O(|\mathcal{V}|^L/n)$ \cite{Shtarkov88,825803}, where the token sequence is generated by an (unknown) $L$-order Markov chain; which means that for sources with large memory $L$ and token alphabet size, we need an extremely long sequence to obtain a good estimate, something that is often not feasible.  \\

\textbf{Identity Testing:} Identity testing, also known as goodness-of-fit, deals with the following hypothesis testing problem: Given an \textit{explicit} description of a probability distribution $p$, samples from an unknown distribution $q$, and bound $\epsilon > 0$, distinguish with high probability between hypothesis $p=q$ and hypothesis $\left\lVert q-p\right\rVert_1 > \epsilon$ whenever $q$ satisfies one of these conditions\cite{acharya2015optimal,diakonikolas2019optimal,CanonneJKL22,pmlr-v108-wolfer20a}. Identity testing presents a potential method of answering the question 'Is this data valuable to me?' where the threshold for "valuableness" is dictated by the constant $\epsilon$. \\

However, the sample complexity required to solve the identity testing problem severely diminishes the practical utility of this method, especially in the context of large language models where $p$ can be well approximated by a high-order Markov chain. In this setting, identity testing requires datasets containing most of the states in the chain, a demand that far exceeds the feasibility of any realistic dataset. Specifically, the requisite number of samples needed to solve the identity testing problem scales as $\Omega\left(\mathcal{V}^{L}/\epsilon^2\right)$\cite{WOLFER202385}. Nevertheless, we leverage some insights from distribution testing to propose our own approach for measuring value.

\section{Value function and its properties } \label{sec:main}

{
The hardness of the identity testing problem for complicated sources, e.g., natural languages, originates from the dependency among the samples and the ever-changing next token probabilities as the context grows. To this end, we propose to use Rosenblatt's transformation \cite{rosenblatt1952remarks} 
which decouples the dependency among tokens and transforms the changing probabilities into the fixed uniform distribution. Rosenblatt's transformation serves as an extension to the probability integral transform (PIT), originally designed for univariate i.i.d. samples. It can be viewed as iteratively applying the PIT using the chain rule in probability (see Appendix \ref{sec: PIT}).

\begin{theorem}[Rosenblatt's transformation \cite{rosenblatt1952remarks}]
\label{thm: multivariate_pit}
Let $X=(X_1,X_2,\ldots,X_d)$ be a continuous $d$-dimensional random vector with a joint density function 
\begin{equation}
    f(x_1,\ldots,x_d) = f_1(x_1)f_2(x_2\mid x_1)\cdots f_d(x_d\mid x_1,\ldots,x_{d-1}).
\end{equation}
Let $F_i(\cdot|x_1,\ldots,x_{i-1})$ be the conditional cumulative distribution function corresponding to the conditional density function $f_i(\cdot|x_1,\ldots,x_{i-1})$. The random variables $Z_1, \ldots, Z_d$ given by Rosenblatt's transformation 
\begin{equation}
\begin{aligned}
Z_1 &= F_1(X_1), \\
Z_i &= F_i(X_i\mid X_1,\ldots,X_{i-1}), \quad i=2,\ldots,d, 
\end{aligned}
\label{eqn: Rosenblatt-Transform}
\end{equation}
are independent and identically distributed following standard uniform distribution $\mathcal{U}$. 
\end{theorem}

Theorem \ref{thm: multivariate_pit} establishes that Rosenblatt's transformation provides an effective and unified approach for testing the plausibility of data. For example, given a random sequence $X=(X_1,\ldots,X_d)$ and a reference distribution $F^{\text{ref}}$, we can obtain $Z=(Z_1,\ldots,Z_d)$ by Rosenblatt's transformation of $X$ according to $F^{\text{ref}}$ as in Equations \eqref{eqn: Rosenblatt-Transform}. Once $X$ follows the reference distribution $F^{\text{ref}}$, the random vector $Z$ will necessarily follow a uniform distribution over the $d$-dimensional hypercube $[0,1]^d$; otherwise, it will not. We thus have an equivalent statement for the identity hypothesis testing:
\begin{equation}
\begin{cases}
H_0: \ X \sim F^{\text{ref}}\\
H_1: \ X \not\sim F^{\text{ref}}
\end{cases} 
\qquad \Leftrightarrow \qquad
\begin{cases}
H_0: \ Z \sim \text{Unif}([0,1]^d)\\
H_1: \ \text{otherwise}
\end{cases} 
\end{equation}
Note that the testing on the right-hand side can be further decomposed into two terms: first, testing the closeness, which can be measured by any $f$-divergence, of the averaged marginal distributions of $Z_1, \ldots, Z_d$ to standard uniform distribution $\mathcal{U}$. Second, testing the independence among $Z_1, \ldots, Z_d$. Both of these are simpler and more standard tests in statistics. Motivated by this, we propose the Uniform-Marginal and Independence (UMI) value function that encapsulates both tests. 

\subsection{The Components of the Value Function}

\begin{wrapfigure}{r}{0.4\textwidth}
\vspace{-0.5cm}
    \centering
    \includegraphics[width=1\linewidth]{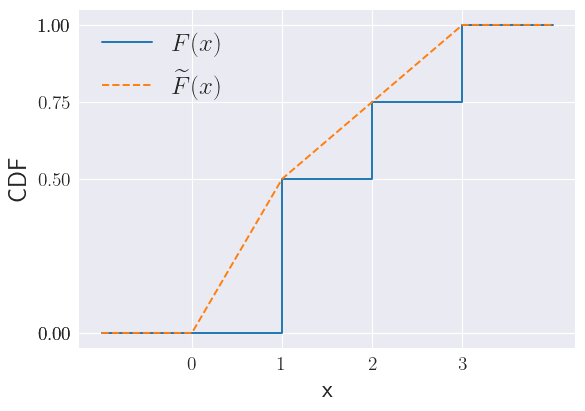}
    \vspace{-0.5cm}
    \caption{A discrete CDF and its interpolated counterpart.}
    \label{fig:discrete_continuous}
    \vspace{-.6cm}
\end{wrapfigure}

\textbf{From Discrete to Continuous.} Rosenblatt’s transformation in Eq \eqref{eqn: Rosenblatt-Transform} can only be applied to continuous random variables.\footnote{Continuity is required everywhere except on a set of Lebesgue measure $0$, commonly referred to as almost everywhere.} However, given that the token vocabulary and probability mass functions generated by the language model are discrete, we adapt by transforming both into continuous forms. Since disparate discrete outcomes are mapped to disjoint intervals on the real line, our choice of transformation will have no effect on what the value function is trying to capture (see Theorem \ref{thm: simple_case}). \\

Let $X$ be a random variable taking values in the discrete set $[k]=\{1,\ldots,k\}$ with a probability mass function $p$. We construct a continuous random variable $\tilde{X}$ that takes values in the range $(0,k)$ through the following procedure: draw the discrete random variable $X$ according to $p$, draw a uniform random variable $U$, and set $\tilde{X}=X-U$. In other words, $\tilde{X}$ will uniformly assume a value in the range $(X-1,X)$. The probability density function of $\tilde{X}$ will be then given by $\tilde{f}(y) = p(\lceil y \rceil)$ for $0<x\leq k$. The CDF of $\tilde{X}$, $\tilde{F}$, will be the linear interpolation of $F$, the CDF of the original variable $X$ (Figure \ref{fig:discrete_continuous}). 

\paragraph{Comparing Against the Uniform.} We employ $f$-divergences to quantitatively assess the disparity between the resultant marginal distribution post-Rosenblatt's transformation and the uniform distribution. $f$-divergences represent a diverse set of functionals which include measures such as total variation distance and Kullback-Leibler divergence. Additionally, we establish theoretical results linking the $f$-divergence between model and data distributions to that of the marginal and uniform distributions. We formally define $f$-divergences below. 
\begin{definition}
\label{def:fDiv}
Let $P$ and $Q$ be two probability distributions over $\mathbb{R}^d$ such that $P$ is absolutely continuous with respect to $Q$ ($P\ll Q$), and let $f:(0,\infty)\to\mathbb{R}$ be a convex function with $f(1)=0$, then $f$-divergence $D_f(\cdot\| \cdot)$ is a functional that maps the pair of distributions into a non-negative number. It is given by
\begin{equation}
D_f(P\mid\mid Q) = \mathbb{E}_Q\left[f\left(\frac{\mathrm{d}P}{\mathrm{d}Q}\right)\right],
\end{equation}
where $\mathrm{d}P/\mathrm{d}Q$ is the Radon-Nikodym derivative of $P$ with respect to $Q$.
\end{definition}
For simple cases (all the cases studied in this work fall into this category), absolute continuity means $\mathrm{Supp}(P) \subset \mathrm{Supp}(Q)$, and $\mathrm{d}P/\mathrm{d}Q$ is just the likelihood ratio of the two distributions. We obtain the total variation distance when $f(x)=1/2\left\lvert x-1\right\rvert$ and the Kullback-Leibler divergence when $f(x)=x\log x$. 


\paragraph{Independence Testing.} Theorem \ref{thm: multivariate_pit} does not only assert a sequence of uniform random variables, but also guarantees their independence. However when dealing with tokens that were not generated by the model, there is a possibility of encountering variables with a nearly uniform marginal distribution that are nonetheless dependent. The goal of independence testing is to detect such cases. \\

To give an example of the tests we have used, we describe the maximum-of-$t$ test \cite{10.5555/270146}. Given a sequence of input numbers $z_1,z_2,\ldots,z_{mt}$, we divide the sequence numbers into $m$ groups of $t$ elements each, that is, $(z_{jt},\ldots, z_{jt + t - 1})$ for $0\leq j\leq m$. We find the maximum of each group to obtain $y_1,\ldots,y_m$. We then apply a Kolmogorov-Smirnov test to the values $y_1,\ldots, y_n$ with the distribution function $F(x)=x^t$ for $0\leq x \leq 1$.\footnote{The Kolmogorov-Smirnov test is used to test whether a sample came from a reference one-dimensional probability distribution.} For a sequence $Z_1,\ldots,Z_t$ of independent variables with a common distribution function $G(z)$, the probability that $\max(Z_1,\ldots,Z_t)\leq z$ is the probability that $Z_i\leq z$ for $1\leq i \leq t$, which is just the product of the individual probabilities $G(z)G(z)\cdots G(z) = G(z)^t$ (see Appendix \ref{app:statistical-test} for a discussion on other independence tests). 

\begin{algorithm}[t]
\caption{Maximum-of-$t$ Test}\label{euclid}
\begin{algorithmic}[1]
\State \textbf{Input:} Sequence of numbers, $Z_1,\ldots,Z_n$
\Statex \hspace{\inputlength} Group size, $t=2,3,\ldots$
\Statex \hspace{\inputlength} $p$-value, $p>0$ 
\For{$i=1,\ldots,\left\lfloor n/t\right\rfloor$}
\State Find $V_j = \max \{Z_{tj}, Z_{tj+1}, \ldots, Z_{tj + t - 1}\}$.
\EndFor
\State Apply the Kolmogorov-Smirnov test to $V_1,\ldots,V_{\lfloor n/t\rfloor}$ with distribution $F(x)=x^t$. 
\State \Return True if $p$-value$>p$ else False
\end{algorithmic}
\end{algorithm}


\subsection{The UMI Value Function}
Given a sequence of tokens $(x_1,\ldots,x_n)$, we first calculate a sequence of probability mass functions $p(\cdot\mid x_{1:i-1})$ using our language model. We compute the continuous versions of the tokens $\tilde{x}_1, \ldots, \tilde{x}_n$ by taking $\tilde{x}_i \sim \text{Unif}(x_i-1, x_i)$, and the continuous CDFs $\tilde{F}(\cdot|x_{1:i-1})$ corresponding to $p(\cdot| x_{1:i-1})$. Note that $x_i = \lceil \tilde{x}_i \rceil$, so we can equivalently write $\tilde{F}(\cdot|\tilde{x}_1,\ldots,\tilde{x}_n)$ without ambiguity. 
Given the continuous random vector $(\tilde{x}_1,\ldots,\tilde{x}_n)$ and reference CDF $\tilde{F}$, Rosenblatt's transformation can be applied as
\begin{align}
\label{eq: cont_cdf}
z_i&=\tilde{F}(\tilde{x}_i \mid x_{1:i-1}), \notag\\
&= \sum_{j=1}^{\lfloor \tilde{x}_i\rfloor} p(j\mid x_{1:i-1}) + (\tilde{x}_i-\lfloor \tilde{x}_i \rfloor)p(\lfloor \tilde{x}_i \rfloor \mid x_{1:i-1}), \notag\\
&= \sum_{j=1}^{x_i} p(j\mid x_{1:i-1}) + u_ip(x_i \mid x_{1:i-1}),
\end{align}
where $u_i=\tilde{x}_i- x_i$. We then conduct marginal distribution test and independence tests on $(z_1, \ldots, z_n)$. For the marginal distribution test, we first calculate the empirical distribution given by $G_n = 1/n \sum_{i=1}^{n} \mathbbm{1}_{Z_i\leq z}$, where $\mathbbm{1}$ is the indicator function. We can compute its interpolated version $\tilde{G}_n$ by linearly interpolating the adjacent jumps in $G_n$. The difference between $\tilde{G}_n$ and the reference standard uniform distribution $\mathcal{U}$ is measured via an $f$-divergence, $D_f(\tilde{G}_n\|\mathcal{U})$ \cite{polyanskiy2024information}. \\

If the computed $f$-divergence $D_f(\tilde{G}_n\|\mathcal{U})< \epsilon$ for some small constant, indicating a near-uniform marginal distribution, we run multiple independence tests to make sure our data does not fall into the edge case where the marginal is uniform, but the $z_i$'s lack independence. If the independence tests fail, then we assign the datapoint a fixed value of $\alpha>0$. Here, $\epsilon$ and $\alpha$ are user-defined hyperparameters.

Let $\mathrm{Independent}(Z)$ be the result of the independence test, returning True if the test passes and False if it fails. The Uniform-Marginal and Independence (UMI) value function can then be written as: 
\begin{align}
    \text{UMI}(x_{1:n}) = \alpha + \mathbf{1}_{\neg\text{Independent}(Z) \lor D_f( \mathcal{U}\| \tilde{G}_n) \geq \epsilon} (D_f(\tilde{G}_n\|\mathcal{U}) - \alpha),
\end{align}
It is easy to see that the proposed UMI value function satisfies the desired properties discussed in the previous sections. In particular, our method does not require any training, only evaluating the model for the input sequence.

}

\subsection{Analysis}

We relate the $f$-divergence between the marginal distribution of $z$ and uniform distribution with the $f$-divergence between the data distribution and the language model.

\subsubsection{IID Case}
Let $P$ and $Q$ be two probability distribution over the finite set of tokens $\mathcal{V}$, and let $F_p$ and $\tilde{F}_p$ (resp. $F_q$ and $\tilde{F}_q$) be the cdf corresponding to $P$ ($Q$) and its continuous counterpart, respectively. Suppose $X_q$ is a r.v. distributed according to $F_q$ and $\tilde{X}_q$ according to $\tilde{F}_q$. Let $Z = \tilde{F}_{p}(\tilde{X}_q)$ be the Rosenblatt's transformation of $\tilde{X}_q$ by a possibly mismatched CDF $\tilde{F}_{p}$ and let $G$ be its CDF. Note that $G = \mathcal{U}$ if $P = Q$. We show below the equivalence of their $f$-divergences, which indicates the measure under the marginal distribution of $Z$ can indeed quantitatively capture the distance between the distributions in the token space. 
\begin{theorem}
\label{thm: simple_case}
Let $P$, $Q$ and $G$ be as described above, then for any $f$-divergence, we have 
\begin{equation}
\label{eq: sm_thm}
    D_f(Q || P) = D_f( G \| \mathcal{U}).
\end{equation}
\end{theorem}
\begin{proof}
The proof uses two properties of $f$-divergences which we state and prove.  
\begin{lemma}
\label{lem:prop}
Let $X$ and $Y$ be any two random variables. Let $P_{X,Y} = P_XP_{Y\mid X}$ and $Q_{X,Y}=Q_XQ_{Y\mid X}$ be any two joint distributions, then 
\begin{align}
D_f\left(P_{X,Y} \mid\mid Q_{X,Y}\right) \geq D_f\left(P_{X} \mid\mid Q_{X}\right).    
\end{align}
\end{lemma}
\begin{proof}
We have
\begin{align}
D_f\left(P_{X,Y} \mid\mid Q_{X,Y}\right) &= E_{X\sim Q_X}E_{Y\sim Q_{Y\mid X}}\left[f\left(\frac{\mathrm{d}P_XP_{Y\mid X}}{\mathrm{d}Q_XQ_{Y\mid X}}\right)\right] \\
&\geq E_{X\sim Q_X}\left[f\left(E_{Y\sim Q_{Y\mid X}}\frac{\mathrm{d}P_XP_{Y\mid X}}{\mathrm{d}Q_XQ_{Y\mid X}}\right)\right] \label{eql1:2}\\
&= E_{X\sim Q_X} \left[f\left(\frac{\mathrm{d}P_X}{\mathrm{d} Q_X}\right)\right]\label{eql1:3} \\
&=D_f(P_X\mid\mid Q_X)
\end{align}

We use the convexity of $f$ and Jensen's inequality in Eq. \ref{eql1:2}. In Eq. \ref{eql1:3}, we use

\begin{align}
E_{Y\sim Q_{Y\mid X}}\left[\frac{\mathrm{d}P_XP_{Y\mid X}}{\mathrm{d}Q_XQ_{Y\mid X}}\right] &= \int \frac{p_X(x)p_{Y\mid X}(y)}{q_X(x)q_{Y\mid X}(y)} Q_{Y\mid X}(\mathrm{dy}) \\
&= \frac{p_X(x)}{q_X(x)}\int \frac{p_{Y\mid X}(y)}{q_{Y\mid X}(y)} q_{Y\mid X}(y) \mu(\mathrm{d}y) \\
&= \frac{p_X(x)}{q_X(x)}\int p_{Y\mid X}(y)\mu(\mathrm{d}y) \\
&= \frac{p_X(x)}{q_X(x)}.
\end{align}
\end{proof}
\begin{lemma}
\label{lem:channel}
Consider a channel that produces $Y$ given $X$ based on the conditional law $P_{Y\mid X}$. Let $P_Y$ and $Q_Y$ denote the distributions of $Y$ when $X$ is distributed as $P_X$ and $Q_X$ respectively, then for any $f$-divergence, we have
\begin{align}
D_f(P_Y\mid\mid Q_Y) \leq D_f(P_X\mid\mid Q_X).    
\end{align}
\end{lemma}



\begin{proof}
Let $P_{X,Y} = P_XP_{Y\mid X}$ and $Q_{X,Y}=Q_XP_{Y\mid X}$. Since the two joint distributions share the same conditional law, we get that $D_f\left(P_{X} \mid\mid Q_{X}\right) = D_f\left(P_{X,Y} \mid\mid Q_{X,Y}\right)$. Now, we use Lemma \ref{lem:prop} to obtain the result. 
\end{proof}   
We obtain Eq. \ref{eq: sm_thm} by applying Lemma \ref{lem:channel} twice. Once with $X$ as the token and $Y= \tilde{F}_{p}(\tilde{X})$  which gives $ D_f(G \|\mathcal{U}) \leq D_f(Q \| P)$, and another time with the roles of $X$ and $Y$ reversed to obtain $D_f(Q \| P) \geq D_f(G \|\mathcal{U})$. Note that we can reverse the roles of $X$ and $Y$ as we can recover the original token through $\left\lceil \tilde{F}_p^{-1}(\tilde{F}_{p}(\tilde{X}))\right\rceil$. Combining the two inequalities, we obtain the desired result. 
\end{proof}

\textbf{Remark:} The empirical distribution $G_n$ satisfies $\| G_n - G \|_\infty \leq \sqrt{\frac{\ln(2 / \delta)}{2n}}$ with probability at least $1 - \delta$ by Dvoretzky–Kiefer–Wolfowitz inequality \cite{dvoretzky1956asymptotic,massart1990tight}. By the continuity of $G$, we can similarly obtain that $\| \tilde{G}_n - G \|_\infty \leq 2\sqrt{\frac{\ln(2 / \delta)}{2n}}$ with probability at least $1 - \delta$, for the interpolated $\tilde{G}_n$ used in the calculating the value. 


\subsubsection{Markov Case}
The stochasticity of Language can be approximated via $m$-gram Markov model \cite{shannon1948mathematical} and the approximation is more accurate for larger $m$. We illustrate the proposed method with measure $D_f(\tilde{G}_n\| \mathcal{U})$ gives a sufficient condition for distinguishing Markov sources. \\ 

Let $P$ and $Q$ be two transition kernels for $m$-gram Markov chains over the finite token set $\mathcal{V}$. Suppose the Markov processes determined by $P$ and $Q$ converge to their unique stationary distributions, and denote by $p^\infty$ and $q^{\infty}$ their stationary distribution over $\mathcal{V}^m$, respectively. \\

Let $(X_{p, t})_{t \geq 1}$ (and $(X_{q, t})_{t \geq 1}$) be a Markov processes generated according to $P$ (respectively $Q$). Let $Y_{p, t} = X_{p, t} - U_t$, where $U_t$ is standard uniform random variable independent of $(X_{p,t})$ (Similarly for $Y_{q, t}$). For $t > m$, define $Z_t = F_{Y_{p, t} | X_{p, t-1}, \ldots, X_{p, t-m}}(Y_{q, t} | X_{q, t-1}, \ldots, X_{q, t-m})$, where $F_{Y_{p, t} | X_{p, t-1}, \ldots, X_{p, t-m}}$ is the cumulative distribution function of $Y_p$ conditioned on $(X_{p, t-1}, \ldots, X_{p, t-m})$. Since $P$ is a transition kernel for $m$-gram Markov chain, we know $F_{Y_{p, t} | X_{p, t-1}, \ldots, X_{p, t-m}}$ is independent of $t$ and we can thus denote it by $F_{Y_p | m}$, which is determined by $P$.

\begin{theorem}
\label{thm: markov_case}
Let $G_t$ be the cumulative distribution function of $Z_t$, then $\lim_{t \rightarrow \infty} \frac{1}{t} \sum_{\tau = 1}^t G_\tau(\cdot)$ exists and converges to $\sum_{ x_{1:m} \in \mathcal{V}^m} q^\infty( x_{1:m} ) F_{Y_p | k}( \cdot | x_{1:m})$ almost surely. We also have that
\begin{equation}
    \sum_{ x_{1:k} \in \mathcal{V}^m} q^\infty( x_{1:m} ) D_f( Q(\cdot | x_{1:m}) \| P(\cdot | x_{1:m})) \geq D_f\left( \lim_{t \rightarrow \infty} \frac{1}{t} \sum_{\tau = 1}^t G_\tau(\cdot) \|\mathcal{U}   \right).
\end{equation}
\end{theorem}
\begin{proof}
Since Markov chain $\{X_{q, t}\}$ converges, we know that for any $x_{1:m} \in \mathcal{V}^m$, $\lim_{t \rightarrow \infty} \frac{1}{t} \sum^{t}_{\tau = 1}\mathbf{1}(x_{1:m} = X_{q, \tau+1:\tau + m}) = q^\infty(x_{1:m})$ almost surely. It thus follows that 
\begin{align*}
    \lim_{t \rightarrow \infty} \frac{1}{t} \sum_{\tau = 1}^t G_\tau(\cdot) & = \lim_{t \rightarrow \infty} \frac{1}{t} \sum_{\tau = 1}^t F_{Y_p | m}(\cdot | X_{q, \tau+1:\tau+m}) \\
    & = \lim_{t \rightarrow \infty} \frac{1}{t} \sum_{\tau = 1}^t 
 \sum_{x_{1:m} \in \mathcal{V}^m} \sum_{\tau = 1}^t F_{Y_p | m}(\cdot | x_{1:m}) \mathbf{1}( x_{1:m} = X_{q, \tau+1:\tau + m} ) \\
 & = \sum_{x_{1:m} \in \mathcal{V}^m} F_{Y_p | m}(\cdot | x_{1:m}) \left( \frac{1}{t} 
 \sum_{\tau = 1}^t \mathbf{1}( x_{1:m} = X_{q, \tau+1:\tau + m} )\right).
\end{align*}
We then have $\lim_{t \rightarrow \infty} \frac{1}{t} \sum_{\tau = 1}^t G_\tau(\cdot) = \sum_{ x_{1:m} \in \mathcal{V}^m} q^\infty( x_{1:m} ) F_{Y_p | m}( \cdot | x_{1:m})$. 
\begin{align*}
    D_f\left( \sum_{ x_{1:m} \in \mathcal{V}^m} q^\infty( x_{1:m} ) F_{Y_p | m}( \cdot | x_{1:m}) ||  \mathcal{U} \right) & \leq \sum_{ x_{1:m} \in \mathcal{V}^m} q^\infty( x_{1:m} ) D_f\left( F_{Y_p | m}( \cdot | x_{1:m}) || \mathcal{U} \right) \\
    & = \sum_{ x_{1:m} \in \mathcal{V}^m} q^\infty( x_{1:m} ) D_f( Q(\cdot | x_{1:m}) || P(\cdot | x_{1:m})),
\end{align*}
where the first relation is by Jenson's inequality and the second relation is calling the equality from the memory-less case. 
\end{proof}
This theorem implies that the divergence measure we use in the proposed algorithm is an \emph{asymptotically sufficient condition} for checking the difference between two Markov models $P$ and $Q$.

\section{Experiments}
\label{sec: experiments}
\paragraph{Model and datasets.} We now study the behavior of our value function on real-world large language models and datasets. In particular, we employ LLaMA2-7B as our language model, configured with a maximum context length of $L=512$. We use $\epsilon=0.05$ and $\alpha=0.1$ as the hyperparameters of the UMI function, and evaluate it over four distinct categories of data: (1) data generated by the model, (2) data generated by the same model using different parameters and sampling methods, (3) tokens and (ascii) characters generated uniformly at random, and (4) new data previously unseen by the model. While we intended to also evaluate our value function on data from the model's training set, the lack of publicized training data prevents us from confidently asserting which data was included. However, we can evaluate the model's performance on texts it has not encountered before, such as publicly available articles written after the model's publication. 

\paragraph{Sampling Methods.} In our analysis, we focused on multinomial (standard) sampling, where the next token is chosen according to the probability distribution provided by the model. However, practitioners use many other sampling methods to generate text from a language model. Thankfully, this does not change the conclusions of our analysis as all other sampling methods can be reduced to multinomial sampling on some transformed probability distribution e.g., greedy sampling, where the next token is always chosen to be the one with the highest probability, can be considered to be a multinomial sampling method with a distribution which places a probability of $1$ on the token with the highest probability as predicted by the model. In our experiments, we use a temperature $T=0.6$ and top-$p$ sampling with $p=0.9$. \\
\begin{table}[t]
\label{tb:nums}
\centering
\begin{tabular}{|l|c|c|}
 \thickhline
 Dataset& Avg. Size & Value \\
 \hline
 Generated by top-$p$   & 1000  &0.0092\\ \\[-1em]
 Generated by top-$k$   & 1000  &0.0163 \\ \\[-1em]
 Different Temp. $T$ & 1000 & 0.0185 \\ \\[-1em]
 \hline
 \hline \\[-1em]
 Random Tokens &  2499  & 0.2617    \\ \\[-1em]
 Random Characters &  7739  & 0.1730   \\ \\[-1em]
 \hline
 \hline \\[-1em]
 New Unseen Data & 5620 & 0.3352 \\
 \thickhline
\end{tabular}
\vspace{.5em}
\caption{\textbf{Evaluating the value function on real-world data.} We report the values given by the UMI value function for different categories of data, alongside the average sequence sizes within each category.}

\end{table}
\begin{figure}[t]
    \centering
    \subfigure[]{\includegraphics[width=0.3\textwidth]{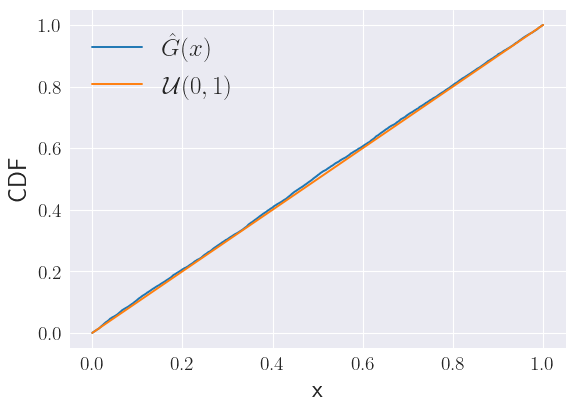}}
    \subfigure[]{\includegraphics[width=0.3\textwidth]{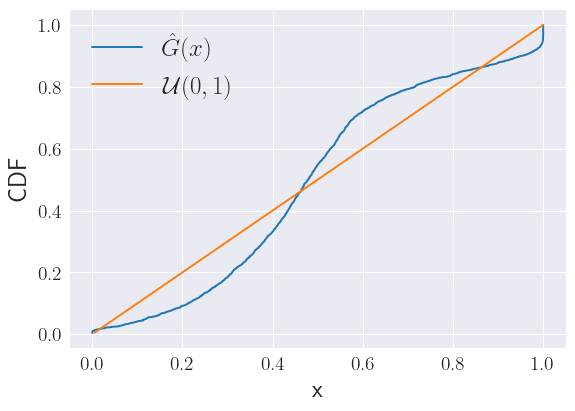}}
    \subfigure[]{\includegraphics[width=0.3\textwidth]{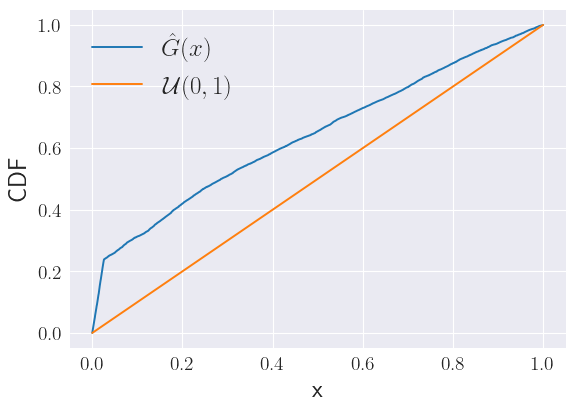}} 
    \subfigure[]{\includegraphics[width=0.3\textwidth]{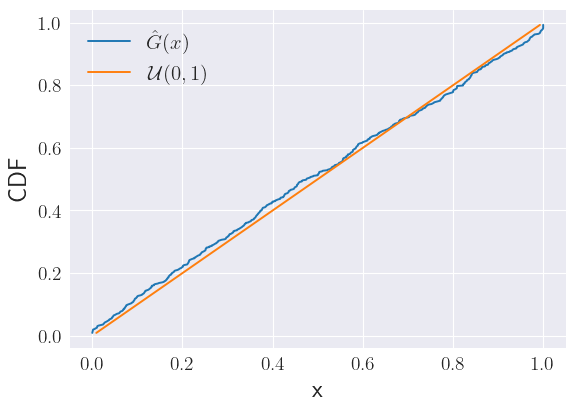}}
    \subfigure[]{\includegraphics[width=0.3\textwidth]{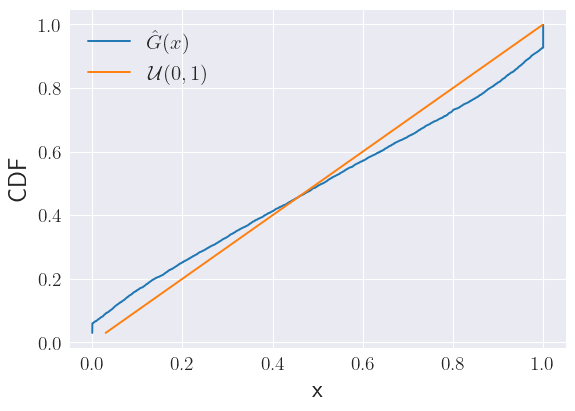}}
    \subfigure[]{\includegraphics[width=0.3\textwidth]{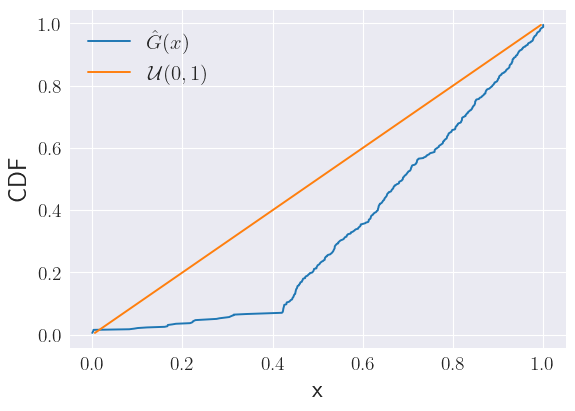}}
    \caption{\textbf{Marginal cumulative distribution of data.} For different data, we plot the obtained marginal cumulative distribution function alongside the cumulative distribution function of the standard uniform for (a) data generated by the model, (b) random tokens, (c) random characters, and (d-f) new unseen data.}
    \label{fig: cdf_visualize}
\end{figure}

\textbf{Value of data generated by the model.} Our UMI function consistently assigns a very low value to the data generated by the model. This holds true regardless of variations in the sampling method employed (top-$k$ with $k=5$) compared to the one reported to the UMI value function (top-$p$), as well as variations in the temperature $T$. This shows that our value function is robust when the specific sampling method and parameters are unknown (see the first 3 rows of Table 1). \\

\textbf{Value when the context is unknown.} Language models are often used by first presenting them with a prompt and then receiving a response. It is not uncommon to run into scenarios where we only have access to the model's response, without knowledge of the prompt it was given. Therefore, we study the behavior of our value function when it is only provided with the model's response. In particular, given a prompt $x_{1:p}$ and response $x_{p+1:p+j}$, we evaluate our UMI function only on $x_{p+1:p+j}$ for $j\leq n$. We observe that the UMI function quickly discerns that the tokens originate from the model despite the missing prompt (or context) and assigns it a low value (see Figure \ref{fig:no_context} in the Appendix).  \\

\textbf{Value of unseen data.} We evaluate the value function on texts written after the model's publication. Even though the model has never seen these texts (and some of them reference very specific recent events), it does not mean that it will assign a high value as the text might be largely predictable given a few tokens (see Figure \ref{fig: cdf_visualize} (d) and (e) for example). \\

\textbf{Visualizing the data.} Our UMI framework enables us to map our data into functions, specifically distributions, over the interval $(0,1)$. By plotting the obtained marginal distribution and comparing it to the baseline uniform distribution, we can visualize the differences between the data and the model represented. \\
\section{Related Works}
Over the past several years, there has been a significant line of research in data valuation, primarily focusing on discriminitive models from a training and optimization perspective. In \cite{pmlr-v97-ghorbani19c}, the 'Data Shapely' is introduced for the purpose of data valuation, where the value of a datapoint is the impact on the performance of the model when that datapoint is removed from the training set. Similarly, \cite{jiang2021characterizing} defines the consistency score (C-score) of a datapoint as the expected of a held-out datapoint given training sets sampled from a data distribution. However, to estimate the data shapely and C-score of datapoints, the model needs to be trained multiple times. To reduce computational costs, \cite{pmlr-v89-jia19a} introduces more efficient algorithms for computing the data shapely when the learning algorithm satisfies certain assumptions, and suggests estimating the data shapely through influence functions when the loss function is smooth. \cite{ki2023data} introduces the complexity-gap-score, a training-free quantity which acts as a proxy for the impact of individual datapoints in the optimization and generalization of classification deep neural networks.\cite{NEURIPS2021_59a3adea} avoids retraining by computing the data diversity, a data-dependent characteristic of the data itself, and connects the diversity to the learning performance. \cite{pmlr-v162-wu22j} estimates the performance of a network through a generalization bound that does not require training to compute, which is then used to find the value of a datapoint. \cite{just2023lava} estimates the generalization performance of a training set by measuring the class-wise Wasserstein distance between the training and the validation sets. In \cite{yang2024gmvaluator}, the authors introduce a method for data valuation specifically for generative models, where the value of a training datapoint is determined through its similarity to a generated sample, effectively measuring the contribution of that training point to the generated output of the model.

\section{Conclusion}
In this work, we introduced the theoretically-grounded UMI value function, based on Rosenblatt's transformation. We analyzed this function theoretically and proved that it exactly matches the distribution distance between the model and the data in i.i.d setting. Moreover, it lower bounds the distribution distance in the Markovian setting. We evaluated our function on real-world datasets, and have showed its effectiveness in assigning a low value to data generated by the model under various situation. There are several open questions, including whether it is possible to integrate the semantic information of the dataset into our measure, since this is not currently incorporated by our proposed value function.

\bibliographystyle{plain}
\bibliography{main}
\clearpage

\begin{appendices}


\section{Additional Proofs}
\label{app:proofs}
\subsection{Univariate Probability Integral Transform}
\label{sec: PIT}
We include the following theorem and its proof for completeness. 
\begin{theorem*} 
Suppose $X$ is a real-valued random variable with a continuous distribution (cdf) $F_X$, then the random variable defined as $Y:=F_X(X)$ is uniformly distributed on the range $(0,1)$. 
\end{theorem*}
\begin{proof}
We define the generalized inverse of the cdf $F_X$ as 
$$F_X^{-1}(y) = \inf \{ x: F_X(x) \geq y\}.$$ 

When $F_X$ is strictly increasing, the generalized inverse coincides with the standard definition of the inverse i.e. $F_X^{-1}(y) = x\iff F_X(x)=y$. However, if $F_X$ is constant on some interval, then the standard definition fails. The generalized inverse takes care of this by assigning to $F_X^{-1}(y)$ a single value. Now, let $0<y<1$, then 

\begin{align}
\mathrm{Pr}(Y\leq y) &= \mathrm{Pr}(F_X(X) \leq y) \\
&=\mathrm{Pr}(F_X^{-1}F_X(X)\leq F_X^{-1}(y)) \\
&= \mathrm{Pr}(X\leq F_X^{-1}(y)) \\
&= F_X\left(F_X^{-1}(y)\right) \\
&= y.
\end{align}
The justification behind the second line second line is $F_X^{-1}$ is always strictly increasing. The justification behind the third line is a bit tricky. If $F_X$ is strictly increasing, then  $F_X^{-1}F_X(x) = x$. However, if $F_X$ is constant on some interval $I = [a,b]$, then $F_X^{-1}F_X(x) \neq x$ for all $x\in I\setminus \{a\}$. But since $P(X\leq x) = P(X\leq a)$ for all $x\in I$, this ends up making no difference. 
\end{proof}
\subsection{Proof of Transformation Change}
\label{sec: trans_pf}
\begin{lemma}
Let $x\in [k]$, $u\in (0,1)$, $y=x-u$, $p(\cdot \mid x_1^{i-1})$ a probability mass function over $[k]$, and $\tilde{F}(\cdot \mid x_1^{i-1})$ be defined as in \ref{eq: cont_cdf}, then 
\begin{equation}
\tilde{F}(y\mid x_1^{i-1}) = up(x\mid x_1^{i-1}) + F(x\mid x_{1}^{i-1}).
\end{equation}
\end{lemma}
\begin{proof}
Using the definition of $\tilde{F}(y_i \mid x_1^{i-1})$ and the fact that $\lfloor y\rfloor = x-1$ and $\lceil y\rceil = x$ , we have
\begin{align}
 \tilde{F}(y \mid x_1^{i-1})&=\sum_{j=1}^{\lfloor y\rfloor} p(j\mid x_1^{i-1}) + (y-\lfloor y \rfloor)p(\lceil y\rceil \mid x_1^{i-1}) \\
 &= F(x-1\mid x_{1}^{i-1}) + up(x\mid x_{1}^{i-1}).
\end{align}
This proves our assertion. 
\end{proof}
\subsection{Multivariate PIT}
\label{sec: copula}
To extend this theorem to the case of multivariate random variables, a straightforward application of the multivariate cumulative distribution function will not work. To see why, consider the continuous random vector $X=(X_1, X_2)$. Let $F_i$ be the marginal distribution function of $X_i$ for $i=1,2$, and $F$ the joint distribution function. Let $x_1, x_2$ be any two real numbers, then
\begin{align}
F(x_1,x_2) &= \mathrm{Pr}(X_1 \leq x_1, X_2\leq x_2), \\
&= \mathrm{Pr}(F(X_1)\leq x_1, F(X_2)\leq F(x_2)),\\
&= C(F(x_1),F(x_2))
\end{align}
where $C(u_1,u_2)$ is the distribution function of $(U_1,U_2) = (F_1(X_1),F_2(X_2))$. $C$ is referred to as a copula in probability theory because it joins or "couples" the joint distribution function $F$ to its marginal distribution functions $F_1,F_2$. While both $U_1$ and $U_2$ have a marginal uniform distribution, $U_1$ and $U_2$ are generally not independent, and their joint distribution rests entirely on the joint distribution of $X_1$ and $X_2$. For instance, if $X=(Z,2Z)$ for some continuous random variable $Z$, then the corresponding copula can be shown to be $C(u_1,u_2) = \min\{u_1,u_2\}$.

\section{Statistical Tests} \label{app:statistical-test}
\subsection{Testing Uniformity}
\subsubsection{Kolmogorov-Smirnov}
Given n observations $U_1,\cdots,U_n$, we find first compute the empirical cumulative distribution function 
\begin{equation}
    F_n(X) = \frac{1}{n} \sum_{j=1}^{n} 1_{U_i \leq x}
\end{equation}
The Kolmogorov–Smirnov statistic for a given cumulative distribution function $F(x)$ is then 

\begin{equation}
D_n = \sup_x \lvert F_n(x) - F(x) \rvert
\end{equation}

The Wiener process $W_{t}$ is the characterised by the following properties:
\begin{itemize}
\item $W_{0}=0$ almost surely
\item $W$ has independent increments: for every $t > 0$,  $W_{t+u}-W_{t}, u\geq 0$ are independent of the past values $ W_{s}, s < t$ 
\item $W$ has Gaussian increments: $W_{t+u}-W_{t}$ is normally distributed with mean $0$ and variance $u$ i.e. $W_{t+u}-W_{t}\sim \mathcal{N}(0,u)$. 
\item $W$ has almost surely continuous paths: $W_{t}$ is almost surely continuous in $t$.
\end{itemize}

Define the Brownian bridge as the stochastic process whose probability distribution is the conditional probability distribution of a standard weiner process subject to $W_1=0$, so that the process is pinned to the same value at both $t = 0$ and $t = 1$ i.e. $B_t:=(W_{t}\mid W_{1}=0),\ t\in [0,1]$ or in other words: 
\begin{equation}
B_t = W_t - t W_1
\end{equation}

Define the random variable $K=\sup _{t\in [0,1]}\lvert B(t) \rvert$. This random variable has what is called the Kolmogorov distribution: 
\begin{equation}
\text{Pr} (K\leq x)=1-2\sum _{k=1}^{\infty }(-1)^{k-1}e^{-2k^{2}x^{2}}={\frac {\sqrt {2\pi }}{x}}\sum _{k=1}^{\infty }e^{-(2k-1)^{2}\pi ^{2}/(8x^{2})}    
\end{equation}

\begin{theorem}
Under null hypothesis that the sample comes from the hypothesized continuous distribution $F(x)$,    
\begin{equation}
    {\sqrt {n}}D_{n}{\xrightarrow[d]{n\to \infty }}K
\end{equation}
\end{theorem}
The goodness-of-fit test or the Kolmogorov–Smirnov test can be constructed by using the critical values of the Kolmogorov distribution. This test is asymptotically valid when $n\to\infty$. It rejects the null hypothesis at level $\alpha$ if 

\begin{equation}
    \sqrt{n}D_{n}>K_{\alpha},
\end{equation}
where $K_\alpha$ is found from $\text{Pr}(K\leq K_{\alpha })=1-\alpha$.
\subsubsection{Chi-Squared Test}
Let $[p_1,\ldots,p_k]$ be a discrete distribution over the categories $\{1,\ldots,k\}$. Let $Y_1,\ldots,Y_n$ be $n$ samples from a discrete distribution taking values over $\{1,\ldots,k\}$. Let $(O_{1},O_{2},...,O_{n})$ be the count number of samples from a finite set of given categories. They satisfy $\sum_{i=1}^kO_{i}=n$.

The null hypothesis is that the count numbers are sampled from a multinomial distribution $\mathrm{Multinomial} (n;p_{1},...,p_{k})$. That is, the underlying data is sampled IID from a categorical distribution $\mathrm {Categorical} (p_{1},...,p_{n})$ over the given categories.The Pearson's chi-squared test statistic is defined as 

\begin{equation}
V:=\sum _{i}{\frac {(O_{i}-np_{i})^{2}}{np_{i}}}    
\end{equation}
\begin{theorem}
 Under the null hypothesis $V\sim \chi^2_{k-1}$ i.e. a chi-squared distribution with $\nu=k-1$ degrees of freedom.   
\end{theorem}
It is important to note that the chi-squared test is designed for discrete random variables. Let $d$ be some positive integer, and define the new sequence $Y_1,\ldots,Y_n$ as
\begin{equation}
Y_j = \lfloor dU_j\rfloor    
\end{equation}

This is a sequence of integers that purports to be independently and uniformly distributed between $0$ and $d-1$. The number $d$ is chosen for convenience; for example, we might have $d = 64 = 2^6$ on a binary computer, so that $Y_n$ represents the six most significant bits of the binary representation of $U_n$. The value of $d$ should be large enough so that the test is meaningful, but not so large that the test becomes impracticably difficult to carry out.
\subsection{Independence Testing}
\subsubsection{Poker test (Partition test)}
The “classical” poker test considers $n$ groups of five successive integers, $\{Y_{5j},Y_{5j+1},\ldots,Y_{5j+4}\}$ for $0\leq j < n$, and observes which of the following seven patterns is matched by each (orderless) quintuple:
\begin{itemize}
    \item All different: abcde
    \item One pair: aabcd
    \item Two pairs: aabbc
    \item Three of a kind: aaabc
    \item Full house: aaabb
    \item Four of a kind: aaaab
    \item Five of a kind: aaaaa
\end{itemize}
A chi-square test is based on the number of quintuples in each category. A simpler version of this test which is almost just as good is to count a number of distinct values in the set of five. In which case, 

\begin{itemize}
    \item 5 values = all different;
    \item 4 values = one pair;
    \item 3 values = two pairs, or three of a kind; 
    \item 2 values = full house, or four of a kind;
    \item 1 value = five of a kind.
\end{itemize}

This breakdown is easier to determine systematically. In general, we can consider $n$ groups of $k$ successive numbers, and we can count the number of k-tuples with r different values. A chi-square test is then made, using the probability
\begin{equation}
p_r = \frac{d(d-1)\cdots (d-r+1)}{d^k} {k \brace r},    
\end{equation}
where ${k \brace r}$ is the stirling number of the second kind. The value of $p_r$ is usually very low for $r=1$ and $r=2$, so we usually lump those together before applying the chi-squared test. 

\subsubsection{Permutation Test}
Divide the input sequence into n groups of t elements each i.e. $V_j = \max(U_{tj}, U_{tj+1},\ldots, U_{tj+t-1})$ for $0\leq j < n$. The elements in each group can have $t!$ possible relative orderings; the number of times each ordering appears is counted, and a chi-square test is applied with $k = t!$ and with probability $1/t!$ for each ordering.

In order to perform this test, we need a method of indexing permutations. The set $\{0,1,2\}$ can be permuted $3!= 6$ ways. Those $6$ permutations and their lexicographic ranks are: 
\begin{itemize}
    \item $0\to (0, 1, 2)$ 
    \item $1\to (0, 2, 1)$ 
    \item $2\to (1, 0, 2)$ 
    \item $3\to (1, 2, 0)$ 
    \item $4\to (2, 0, 1)$
    \item $5\to (2, 1, 0)$
\end{itemize}
Calculating sequential indexes for permutations is done by computing the Lehmer code of the permutation, and then converting that \textbf{Lehmer code} to a base-10 number. Like that base-10 number, a Lehmer code is just a sequence of digits; However, each digit has a different base. Technically, it’s a mixed-radix numeral system known as a factorial number system.

First, let $\sigma$ be a permutation of $n$ numbers $\{0,\ldots, n-1\}$. For example, if $n=3$, we can represent $\sigma$ as 

\begin{equation}
\sigma = \begin{pmatrix} 0 & 1 & 2 \\ 1 & 2 & 0 \end{pmatrix}    
\end{equation}
This means that $\sigma$ places in the 0th position what used to be in the 1st position, in the 1st position what used to be in the 2nd position, and in the 2nd position what used to be in the 0th position. So $\sigma((0,1,2)) = (1,2,0)$ and $\sigma((2, 0, 1)) = (0,1,2)$. A more compact notation of the above is $\sigma = (\sigma_0,\ldots,\sigma_{n-1})$ which just include the second row of the above representation. The Lehmer code of a permutation $L(\sigma) = (L(\sigma)_0,\ldots,L(\sigma)_{n-1})$ is given by: 
\begin{equation}
L(\sigma)_i = \left\lvert \{ j>i: \sigma_j < \sigma_i\}\right\rvert
\end{equation}

In other words, $L(\sigma)_i$ counts the number of terms to the right of $\sigma_i$ in $(\sigma_0,\ldots,\sigma_{n-1})$ which are smaller than it, and the number is between $0$ and $n-i-1$; A pair of indices $(i,j)$ with $i < j$ and $\sigma_i > \sigma_j$ is called an inversion of $\sigma$, and $L(\sigma)_i$ counts the number of inversions $(i,j)$ with $i$ fixed and varying $j$. For example, $L((1,2,0)) = (1,1,0)$. Moreover, we can easily convert the Lehmer code into the permutation index. The index of the permutation is then $\sum_{i=0}^{n-1} L(\sigma)_i(n-i-1)!$ e.g. the index of $(1,2,0)$ is $1(2!) + 1(1!) + 0(0!) = 3$. 

We can find the Lehmer code in linear time. Let $b=00\cdots0b$ be a binary number with $n$ bits and let $\sigma=(\sigma_0,\ldots,\sigma_{n-1})$ be our permutation, then for $i=0,\ldots,n-1$, we 
\begin{enumerate}
    \item We flip bit the $\sigma_i$th bit of b
    \item Right-Shift $b$ by $n-\sigma_i$ to get $b'$. 
    \item Count the number of 1s in $b'$ and subtract it from $\sigma_i$ to get the Lehmer code $L(\sigma)_i$. 
\end{enumerate}

\textbf{Proof of correctness:} Say we are at position $i$ in the loop. Until now bit $j$ of $b$ is $1$ iff $j$ is to the left of $\sigma_i$. We you shift $b$ by $n-\sigma_i$ positions to the right, and what would be left is the first $\sigma_i$ bits of $b$ corresponding to the numbers $0,\ldots,\sigma_i-1$.  
\subsubsection{Serial Correlation Test}
We may also compute the following statistic:
\begin{equation}
C = \frac{n\left(U_0U_1+U_1U_2+\ldots U_{n-2}U_{n-1} + U_{n-1}U_0\right) - \left(U_0+U_1+\ldots+U_{n-1}\right)^2}{n\left(U_0^2+U_1^2+\ldots+U_{n-1}^2\right) - \left(U_0+U_1+\ldots+U_{n-1}\right)^2}.    
\end{equation}
This is the “serial correlation coefficient,” a measure of the extent to which $U_{j+1}$ depends on $U_j$. Correlation coefficients appear frequently in statistical work. If we have $n$ quantities $U_0, U_1,\ldots, U_{n-1}$ and $n$ others $V_0, V_1,\ldots, V_{n-1}$, coefficient between them is defined to be:

\begin{equation}
C = \frac{n\sum (U_jV_j) - (\sum U_j)(\sum V_j)}{\sqrt{\left(n\sum U_j^2 - (\sum U_j)^2\right)\left(n\sum V_j^2 - (\sum V_j)^2\right)}}
\end{equation}
When $C$ is zero or very small, it indicates that the quantities $U_j$ and $V_j$ are uncorrelated (a weaker notion than independence), whereas a value of $\pm 1$ indicates total linear dependence. Therefore it is desirable to have $C$ close to zero. In actual fact, since $U_0U_1$ is not completely independent of $U_1U_2$, the serial correlation coefficient is not expected to be exactly zero. 

In other words, we are finding the correlation between $(U_0,U_1,\ldots,U_{n-1})$ and $(U_1,\ldots,U_{n-1},U_0)$. We can also compute the correlation coefficient between $(U_0,U_1,\ldots,U_{n-1})$ and any cyclically shifted sequence $(U_q,\ldots,U_{n-1},U_0,\ldots,U_{q-1})$. A naive approach to computing the cyclic correlations takes $\mathcal{O}(n^2)$. However, one can use the Fast Fourier Transform to find all these correlations in $\mathcal{O}(n\log n)$.







\subsubsection{Serial Test}
More generally, we want pairs of successive numbers to be uniformly distributed in an independent manner. We can also use the chi-squared test to test for independence (somewhat). 

To carry out the serial test, we simply count the number of times that the pair $(Y_{2j},Y_{2j+1}) = (q,r)$ occurs, for $0<=j<n$; these counts are to be made for each pair of integers $(q, r)$ with $0\leq q,r < d$, and the chi-square test is applied to these $k = d^2$ categories with probability $1/d^2$ in each category.

As with the equidistribution test, $d$ may be any convenient number, but it will be somewhat smaller than the values suggested above since a valid chi-square test should have $n$ large compared to $k$. A common rule of thumb is to take $n$ large enough so that each of the expected values $np_i\geq 5$ i.e. the expected number of occurrences of each categorical outcomes is five or more; preferably, however, take n much larger than this, to get a more powerful test. Hence, here we need $n\geq 5d^2$. 
\subsubsection{Gap Test}
Another test is used to examine the length of “gaps” between occurrences of $U_j$ in a certain range. Let $\alpha$ and $\beta$ be two real numbers with $0\leq \alpha <\beta \leq 1$, and suppose $U_j \in [\alpha,\beta)$, then we want to consider the lengths of consecutive subsequences $U_j , U_{j+1}, ..., U_{j+r}$ in which $U_{j+r}$ lies between $\alpha$ and $\beta$ but the other $U$’s do not. This subsequence of $r+1$ numbers represents a gap of length r.

The classical gap test considers the sequence $U_1,\ldots,U_n$ to be cyclic sequence with $U_{n+j}$ identified with $U_j$. If $m$ of the numbers $U_1,\ldots,U_n$ fall into the range $[\alpha,\beta)$, there are $n$ gaps in the cyclic sequence. We have the following theorem:

\begin{theorem}
Let $t$ be some positive integer. Let $G_r$ be the counts of the gaps of length $r$ for $0\leq r\leq t-1$, and $G_t$ the count of gaps of length $\geq t$. Define $p=\beta-\alpha$, and let 
\begin{align}
p_r &=p(1-p)^r,  \text{ for  $0\leq r\leq t-1$} \\
pt &=(1-p)^t. 
\end{align}  
then as $n\to\infty$, $V = \sum_{i=1}^t \frac{(G_i - np_i)^2}{np_i}$ follows a chi-squared distribution with $t$ degrees of freedom.
\end{theorem}
The gap test is often applied with $\alpha=0$ or $\beta = 1$ in order to omit one of the comparisons in the algorithm. The special cases $(\alpha,\beta) = (0,1/2)$ or $( 1/2 , 1)$ give rise to tests that are sometimes called “runs above the mean” and “runs below the mean,” respectively.
\subsubsection{Maximum-of-$t$ Test}
For $0\leq j < n$, let $V_j = \max(U_{tj}, U_{tj+1},\ldots, U_{tj+t-1})$. Let $F(x)$ be the cdf of $U_i$, we now apply the Kolmogorov–Smirnov test to the sequence $V_0, V_1,\ldots, V_{n-1}$, with the distribution function $F_{max}(x) = F(x)^t, \ 0\leq x\leq 1$. 

To verify this test, we must show that the distribution function for the $V_j$ is $F_{max}(x) = F(x)^t$. The probability that $\max(U_1, U_2,\ldots, U_t)\leq x$ is the probability that $\textrm{Pr}(U_1\leq x, U_2\leq x,\ldots, U_t\leq x) = F(x)F(x)\cdots F(x) = F(x)^t$.

\subsubsection{Runs Test}
consider the sequence of ten digits “1298536704”. Putting a vertical line at the left and right and between $X_j$ and $X_{j+1}$ whenever $X_j > X_{j+1}$, we obtain

$$|129| 8| 5| 367| 04|$$
This shows that run ups. Unlike the other tests, we should not apply a chi-squared test to the run counts, since adjacent runs are not independent. A long run will tend to be followed by a short run and vice versa. We can use other statistics to take this into account. However, a simple and more practical test is the following: If we “throw away” the element that immediately follows a run, so that when $X_j$ is greater than $X_{j+1}$ we start the next run with $X_{j+2}$, the run lengths are independent, and a simple chi-square test can be used. 

\begin{proposition}
 For a sequence of iid continuous random variables, the probability of run of length $k$ starting is given by 
\begin{equation}
    p_k = \frac{k}{(k+1)!}
\end{equation}   
\end{proposition}
For a run of length $k$, we need a $k+1$ random variables $X_1,\ldots,X_{k+1}$. If the sequence were iid, then we would have $(k+1)!$ possible orderings each with the same probability. The number of these orderings that give us a run of length $k$ is exactly $k$. Since the number in the $k$th position must be the unique maximum of $X_1,\ldots,X_{k+1}$. 

\section{Additional Experimental Details}
\label{app:add_exp}
\begin{figure}
    \centering
    \includegraphics[width=0.45\textwidth]{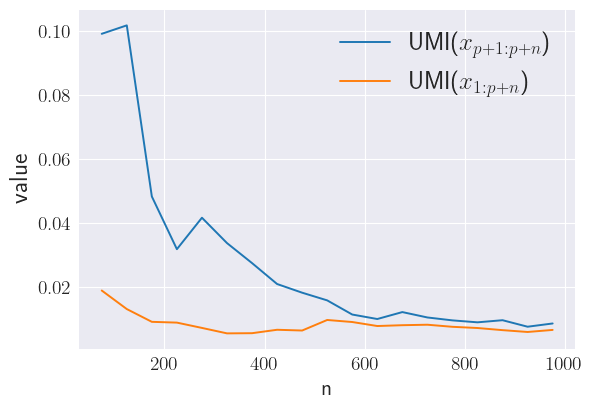}
    \caption{UMI value function when the prompt is not given.}
    \label{fig:no_context}
\end{figure}
\label{sec:exp_data}
\subsection{Temperature}
Recall that the logits $z_i$ produced by the language model (one for each token in our vocabulary) are transformed into probabilities through the softmax function given by: 
\begin{equation}
p_i = \frac{e^{z_i}}{\sum_{j=1}^{n} e^{z_j}},    
\end{equation}
The temperature $T$ is a parameter used to adjust the output probabilities. A high temperature ($T>1$) flattens the probability distribution bringing it closer to a uniform distribution. Meanwhile, a small $T$ exacerbates any leads high probability tokens have making the probability distribution sharper. In other, temperature can be interpreted as the confidence a model has in its most likely responses. With temperature adjustment, the new probabilities are given by
\begin{equation}
p_i = \frac{e^{z_i/T}}{\sum_{j=1}^{n} e^{z_j/T}}.
\end{equation}
\subsubsection{Sampling Methods}
We give a brief description of the most common sampling methods used in text generation using language models. 

\textbf{Greedy Sampling:} The next token is always chosen to be the highest probability token, as predicted by the model. While simple and fast, this sampling method is prone to producing sentences with low probability.This occurs because we might select a high-probability token initially, only to encounter subsequent tokens of considerably lower probabilities. 

\textbf{Beam Search:} This method addresses the shortcomings of greedy sampling by maintaining a fixed number of potential sentences (sequences of tokens of some fixed length) called beams. It then chooses the sentence with the highest probability.

\textbf{Top-$k$:} This method restricts our selection of the next token to the top $k$ tokens with the highest probability, where $k$ is a hyper-parameter. Since $k$ is fixed, this method does not dynamically adjust the number of candidates based on the probability distribution. As a result, some very unlikely tokens may be selected among those top $k$ candidates.

\textbf{Top-$p$:} This method addresses the issues of top-$k$ sampling by selecting as candidates the most probable tokens until their combined probability surpasses a specified threshold $p$. 

As mentioned before, we can reduce each of the above sampling method to multinomial sampling methods with a modified probability distribution. For example, we can consider greedy sampling to be a multinomial sampling method with a distribution which places a probability of $1$ on the token with the highest probability as predicted by the model.
\end{appendices}

\end{document}